\documentclass{article}

\usepackage{microtype}
\usepackage{graphicx}
\usepackage{booktabs} 

\usepackage{hyperref}

\usepackage{url}            
\usepackage{subfig}
\usepackage{color}
\usepackage{amsmath}
\usepackage{bbm}
\usepackage{amssymb}
\usepackage{amsthm}
\usepackage{thmtools, thm-restate}
\newcommand{\algorithmicparameters}{\textbf{parameters}}
\newcommand{\PARAMETERS}{\item[\algorithmicparameters]}

\newcommand\numberthis{\addtocounter{equation}{1}\tag{\theequation}}

\theoremstyle{definition}

\newtheorem{thm}{Theorem}

\newtheorem{lem}{Lemma}

\DeclareMathOperator*{\argmax}{arg\,max}

\DeclareMathOperator*{\expit}{expit}

\newcommand{\Real}{\mathbb{R}}
\newcommand{\norm}[1]{\left\lVert#1\right\rVert}
\newcommand{\wS}{{S_w}}

\newcommand{\DS}{\mathcal{D}_S}
\newcommand{\DT}{\mathcal{D}_T}

\newcommand{\DZ}{\mathcal{D}_Z}
\newcommand{\Ddata}{\mathcal{D}_\text{data}}
\newcommand{\Hfamily}{\mathcal{H}}
\newcommand{\Ffamily}{\mathcal{F}}

\newcommand{\lzo}{l_{\text{0-1}}}
\newcommand{\Comb}{C}

\newcommand{\Xspace}{\mathcal{X}}

\newcommand{\foutcome}{f_Y}

\newcommand{\Mfoutcome}{M_{Y}}

\newcommand{\dHapprox}{d_{\Hfamily}}

\newcommand{\cY}[1]{Y^{#1}}  

\newcommand{\expect}[2]{{\mathop\mathbb{E}}_{#1}\left[#2\right]} 
\newcommand{\expectest}[2]{{\mathop\mathbb{\widehat{E}}}_{#1}[#2]} 
\newcommand{\bigCI}{\mathrel{\text{\scalebox{1.07}{$\perp\mkern-10mu\perp$}}}}

\newcommand{\wa}{\omega^a}

\newcommand{\dist}{\mathcal{D}}
\newcommand{\Dp}{\dist}

\newcommand{\bx}{{\bf x}}
\newcommand{\bw}{{\bf w}}
\newcommand{\bc}{{\bf c}}
\newcommand{\bz}{{\bf z}}
\newcommand{\be}{{\bf e}}
\newcommand{\bu}{{\bf u}}
\newcommand{\lab}{\bc}

\newcommand{\disc}{d} 
\newcommand{\loss}{L}
\newcommand{\lossn}{\loss_n}
\newcommand{\idx}{\mathcal{I}}

\newcommand{\indicator}[1]{\mathbbm{1}\left[#1\right]}

\usepackage[accepted]{icml2019_arxiv}

\icmltitlerunning{Adversarial Balancing for Causal Inference}

\begin{document}

\twocolumn[
\icmltitle{Adversarial Balancing for Causal Inference}




\icmlsetsymbol{equal}{*}

\begin{icmlauthorlist}
\icmlauthor{Michal Ozery-Flato}{equal,hrl}
\icmlauthor{Pierre Thodoroff}{equal,mcgill} 
\icmlauthor{Matan Ninio}{hrl}
\icmlauthor{Michal Rosen-Zvi}{hrl}
\icmlauthor{Tal El-Hay}{hrl}
\end{icmlauthorlist}

\icmlaffiliation{hrl}{IBM Research}
\icmlaffiliation{mcgill}{McGill University; This work was done while the author was an intern in IBM Research-Haifa}

\icmlcorrespondingauthor{Tal El-Hay}{talelh@il.ibm.com}

\icmlkeywords{Machine Learning, ICML}

\vskip 0.3in
]



\printAffiliationsAndNotice{\icmlEqualContribution} 

\begin{abstract}
Biases in observational data of treatments pose a major challenge to estimating expected treatment outcomes in different populations. 
An important technique that accounts for these biases is reweighting samples to minimize the discrepancy between treatment groups. 
We present a novel reweighting approach that uses bi-level optimization to alternately train a discriminator to minimize classification error, and a balancing weights generator that uses exponentiated gradient descent to maximize this error.
This approach borrows principles from generative adversarial networks (GANs) to exploit the power of classifiers for measuring two-sample divergence. 
We provide theoretical results for conditions in which the estimation error is bounded by two factors: (i) the discrepancy measure induced by the discriminator; and (ii) the weights variability.
Experimental results on several benchmarks comparing to previous state-of-the-art reweighting methods demonstrate the effectiveness of this approach in estimating causal effects.
\end{abstract}

\section{Introduction}

Causal inference deals with estimating expected outcomes for treatments or interventions. The gold standard for causal inference studies is randomized controlled trials (RCTs), in which treatment and control groups come from the same data distribution, due to a randomized treatment assignment process. However, RCTs are often costly, sometimes impractical to implement, and may raise ethical questions. An appealing alternative is to infer expected treatment outcomes using the abundance of observational treatment data. Alas,
in such data treatment populations are likely to differ from each other. This difference, or bias, between treatment groups and the lack of knowledge on the treatment assignment mechanism, hinder the inference of  expected outcomes for the treatment in the entire  population.

A common approach for inferring expected treatment outcomes from observational treatment data is by balancing the bias between treatment groups via reweighting of the individuals in these groups.  The challenging task of computing balancing weights has applications not only for causal inference, but also for transfer learning, and is highly related to the field of density ratio-estimation (see Related work section).   

Motivated by the immense success of generative adversarial networks (GANs) in producing simulated data that highly resembles real world samples, we propose a novel framework\footnote{Our code is available at \url{https://github.com/IBM/causallib/tree/master/causallib/contrib/adversarial_balancing}} that adapts the objective of GANs to the task of generating balancing weights. Similar to GANs, our framework is based on a two-player game that involves
a discriminator that measures the bias, or discrepancy, between two data samples, and a generator that aims to produce data indistinguishable, by the discriminator, from another given dataset. The key difference from GANs is that our data generator produces "new" data by reweighting a given dataset, where the weights are obtained by a simple step of exponentiated gradient ascent step on the discriminator loss.  Using this framework allows us to harness the complete arsenal of classification methods to the task of generating balancing weights. We evaluated the performance of this algorithm on a range of published causal-inference benchmarks, and assessed the ability to select an appropriate classifier for the input datasets in a standard cross-validation routine.

\section{Problem setup}

Consider a population where each individual received a single treatment from a finite set of treatments $\mathcal{A}$.  The received treatment and the resulting outcome for every individual are indicated by the variables $A$ and $Y$, respectively. For every treatment $a\in\mathcal{A}$ , $\cY{a}$ denotes the potential outcome for the treatment. The variable $\cY{a}$ is observed only when $A=a$. Let $X$ denote the vector of observed pre-treatment covariates used to characterize the individuals. Let $\Dp$ be the distribution over $(X,A,\{Y^a\}_{a\in \mathcal{A}})$ in the population.  
The expected outcome of a treatment $a\in\mathcal{A}$ in the population is:
\begin{equation}
\expect{\cY{a} \sim \Dp}{\cY{a}} 
=
\expect{X  \sim \Dp(X)}{\expect{\cY{a} \sim \dist(\cY{a}|X)}{ \cY{a}|X}} \,. 
\label{eq:1}
\end{equation}
For brevity, we denote $\expect{}{\cY{a}}\equiv\expect{\cY{a} \sim \Dp}{\cY{a}}$.

The goal of many observational studies is to estimate $\expect{}{\cY{a}} $ from a finite data sample from $\Dp$.  However, $\cY{a}$ is observed only in the subpopulation that actually received treatment $a$, where the distribution over $X$ is $\Dp(X|A=a) \neq  \Dp(X)$.  To overcome this hurdle, we employ the standard assumptions of {\it strong ignorability}: $\cY{a} \bigCI  A | X$, and {\it positivity}: $0<p(A=a|X=x)<1$,  $\forall a \in \mathcal{A}$ \cite{rosenbaum1983central}. Strong ignorability, often stated as "no hidden confounders", means that the observed covariates contain all the information that may affect treatment assignment. These assumptions allow rewriting Equation \ref{eq:1}:
\begin{equation}\label{eq:3}
\expect{}{\cY{a}} = \expect{X \sim \Dp(X)}{\expect{Y\sim \dist(Y|X,A=a)}{Y|X,A=a}}\,.\\ 
\end{equation}
Equation \ref{eq:3} suggests that $\expect{}{\cY{a}}$ can be estimated by a sample from the subpopulation corresponding to $A=a$ under the condition that its distribution over $X$ is $\Dp(X)$.  A common approach to handle this sampling challenge is to use a weighting function $\wa(X)$ such that $\Dp(X|A=a) \wa(X) = \Dp(X)$. The weighting function that satisfies this condition is clearly  $\wa(X) = \frac{\Dp(X)}{\Dp(X|A=a)}=\frac{\Dp(A=a)}{\Dp(A=a|X)}$ and therefore
\begin{equation}\label{eq:ipw_est}
\expect{}{\cY{a}}
= \expect{X \sim \Dp(X|a)}{ \wa(X) \expect{Y\sim \dist(Y|X,A=a)}{Y|X,A=a}} \,.
\end{equation}
\newcommand{\sample}{S}
Given  a finite sample $\sample = \{ (x_i, a_i, y_i)\}_{i=1}^N$ from $\Dp$ we would like to produce weights $w_i$ for each $i \in \{ i:a_i=a\}$ that approximate $\wa(X_i)$. Following Equation \ref{eq:ipw_est}, given such weights we estimate $\expect{}{ \cY{a}}$ by:
\begin{equation}
\label{eq:outcomeEstimator}
\expectest{}{ \cY{a}} = \sum_{i:a_i=a} w_i y_i \,.
\end{equation}

\section{Background on adversarial framework for learning generative models}

The adversarial framework, which was introduced by Goodfellow et al. \cite{goodfellow2014generative}, aims to learn a generative model of an unknown distribution $\Ddata$ using a class of discriminators that gauge the similarity between data distributions. This framework can be described as a game in which a generator simulates data and a discriminator tries to distinguish samples of true data from simulate data samples. The generator employs generative models with an input random variable $Z$ from a predefined distribution $\DZ$ and a deterministic mapping $g(\bz)$ to the data space $\Xspace$. Simulated data are generated by sampling data from $\DZ$ and transforming them through $g$. At the end of each round of the game, the generator observes the predictions of the discriminator and updates the model for $g(\bz)$.  Given the generator model $g(\bz)$ the prediction model of the  discriminator, $\disc(\bx)$, attempts to minimize the expected classification error in the real and simulated samples : 
\begin{align*}\label{eq:AdversarialLoss}
\loss(g, \disc) 
&= 
\expect{\bx\sim \Ddata}{l(\disc(\bx), 1)} + \expect{\bz\sim \DZ}{l(\disc(g(\bz)), 0)} 
\numberthis
\end{align*}
where $l$ is the loss function.  Given the prediction model, $\disc(\bx)$, the generator attempts to maximize the expected error, and its objective is to find
\begin{equation}
g^{*} = \argmax_{g} \left( \min_\disc \loss(g, \disc) \right)
\end{equation}
Examples for loss functions are the Log-loss ${l(\disc(x), c) = c \cdot \log\disc(x) + (1-c)\cdot \log (1-\disc(x))}$ , which was used in \cite{goodfellow2014generative}; and the 0-1 loss \footnote{$\indicator{c}$ is the indicator function which is $1$ is predicate $c$ if true, and $0$ otherwise.   
}, 
$
{
	\lzo(\disc(x),c)=\indicator{\indicator{\disc(x)> \frac{1}{2}}\neq c}
}
$,
which was used in \cite{gutmann2014likelihood, mohamed2016learning} for likelihood-free inference and in \cite{lopez2016revisiting} for classifier two-sample tests. In the next section we adapt the adversarial framework and its key principle of maximizing the discrimination loss to the task of generating balancing weights.

\section{Adversarial balancing weights}

In this section we present our adversarial framework for generating balancing weights, and a novel algorithm that applies it.  Similar to GAN, our goal is to generate a sample that resembles data coming from a distribution $\Dp(X)$.  However, while in the original GAN framework the generated sample is \emph{simulated} by applying a transformation on unlimited random data, our balancing framework is constrained to \emph{reweight} a finite samples from the distribution $\Dp(X|a)$.   More generally we consider the problem of reweighting a data sample coming from a source distribution $\DS$ on $X$ such that it becomes indistinguishable from a sample of a target distribution $\DT$.  
The input to our problem are two finite samples from the two distributions:
$$S = \{ \bx_i\}_{i=1}^n  \sim (\DS)^n ;  \; T = \{ \bx_i \}_{i=n+1}^{n+n'}  \sim (\DT)^{n'} \,.$$

This is a general framework for balancing with respect to any target population; therefore, it can be used to estimate different types of causal effects. For example, the {\it average treatment effect} (ATE) is defined as $\expect{X \sim \Dp(X)}{\cY{a=1}} - \expect{X\sim \Dp(X)}{\cY{a =0}}$. In this case we estimate  $\expect{X\sim \Dp(X)}{\cY{a}}$ using $\DS := \Dp(X|a)$ and $\DT := \Dp(X)$. 
Another example is the {\it average treatment effect in the treated} (ATT), which is defined as $\expect{X \sim \Dp(X|A=1)}{\cY{a=1}} - \expect{X \sim \Dp(X|A=1)}{\cY{a=0}}$.  In the latter example, we estimate $ \expect{X \sim \Dp(X|A=1)}{\cY{0}}$ by reweighting a sample from the distribution $\DS := \Dp(X|A=0)$ and using $\Dp(X|A=1)$ as the target distribution.

\begin{figure*}[tb]
	\includegraphics[scale=.5]{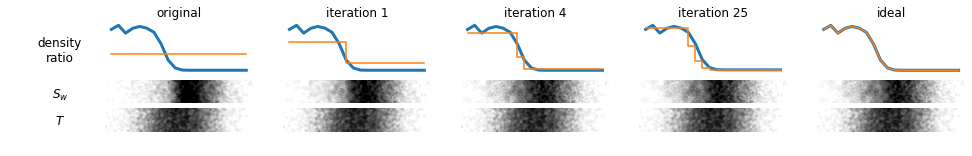}
	\caption{Illustration of the adverasarial balancing algorithm. The thick blue line represents the density ratio $w^*(x)\equiv\frac{\DT(x)}{\DS(x)}$. The thin orange line represents the estimated weights in different iterations. The bottom strip is a scatter plot of samples drawn from the target distribution (the density in the $x$-axis is $\DT$ and uniform in the $y$-axis for visualization purpose). This strip is constant in all iterations. The top strip shows a scatter plot of the source distribution where the size of point is proportional to its weight, thereby visualizing $\wS$. The algorithm starts by uniform weights and a high bias. At each iterations the weights are updated according to a classifier that minimizes $\lossn$ to maximize the loss in the next iteration.
	} 
	\label{fig:illustration}
\end{figure*}

\subsection{Discrepancy objective}

Let $w(X)$ be a non-negative function that reweights samples from $\DS$ resulting in a new distribution, $\dist_\wS(X) = w(X)\DS(X)$.  For $\dist_\wS$ to be a valid distribution, $w(X)$ must satisfy the constraint 
 \begin{equation}
 \label{eq:norm_w_function}
\expect{X \sim \DS}{w(X)}=1\,.
 \end{equation}
Using a similar loss to the one in Equation \ref{eq:AdversarialLoss} and replacing the generator distribution by $\dist_\wS(X)$, we obtain
\begin{equation}
\label{eq:expectedBalancingLoss}
\loss(w, \disc) = \expect{ X \sim \DT} 
{l(\disc(X), 1)} + \expect{X\sim \DS}{w(X) l(d(X), 0)}
\end{equation}

We can confine the representation of $w(X)$ to a family of models and optimize the loss with respect to this family. However, for the estimation problem defined in Equation \ref{eq:outcomeEstimator}, it suffices to infer point estimates of $w(X)$ for the given sample from $\DS$. We denote such point estimates by $w_i \equiv w(x_i)$
and use them in Equation \ref{eq:norm_w_function} to obtain the following normalization constraint: 
\begin{equation}
\frac{1}{n}\sum_{i=1}^n w_i = 1 \,.
\end{equation}
Similar to the constraint in  \ref{eq:norm_w_function}.  The discriminator error becomes 
\begin{equation}
\label{eq:nonParametricEmpiricaLoss}
\lossn(\bw, \disc) = \frac{1}{n'} \sum_{i=n+1}^{n+n'}  l(\disc(X_{i}),1) + \frac{1}{n} \sum_{i=1}^n w_i l(\disc(X_{i}),0) \,,
\end{equation}
where  $\bw\equiv (w_1, \dots, w_n)$. Note that Equation \ref{eq:nonParametricEmpiricaLoss} is the empirical error of the discriminator, when the samples from $\DS$ and $\DT$ are given the same importance. 
The aim of the discriminator is to minimize the error in Equation \ref{eq:nonParametricEmpiricaLoss}. Weights $\bw$ leading to large errors of the discriminator imply its inability to distinguish between the sample $T$ and the weighted samples $S$. We formulate the objective of the adversarial balancing framework as solving the following optimization problem:
\begin{equation}
\label{eq:adv_max_min_obj}
\bw^{*} = \argmax_{\frac{\bw}{n} \in \Delta} \left( \min_\disc \lossn(\bw, \disc) \right)
\end{equation}
where $\Delta$ be the unit simplex $\Delta = \{ \bu \in \Real^n : \bu\geq 0,  \;   \sum_{i=1}^n u_i=1 \}$. 

\subsection{Weight learning algorithm}

To search for a solution to the max-min objective in Equation \ref{eq:adv_max_min_obj}, we propose the following iterative process. At each step we train a discriminator to minimize the empirical loss of Equation \ref{eq:nonParametricEmpiricaLoss}. We then update the weights $w_i$ to increase this loss using a single step of exponentiated gradient descent \cite{kivinen1997exponentiated}, which maintains the weight normalization constraint. Figure \ref{fig:illustration} illustrates this process. 

We define the augmented labeled dataset by assigning a class label $0$ and weights to $\DS$, and a class label $1$ and uniform weights to $\DT$:
\begin{equation}\label{eq:10} 
\big\{ (\bx_i,0; w_i) \big\}_{i=1}^n \bigcup \big\{ (\bx_i,1; w_i=1) \big\}_{i=n+1}^{n+n'} \,.
\end{equation}
Note that the uniform weights we assigned to $\DT$ will not be modified by our algorithm.  The discriminator predicts the class label, $C$, of the samples in $\DT$ using a classifier $\disc(\bx) \in \Ffamily$, where $\Ffamily$ is a predefined classifier family.  Recall that the final objective of the adversarial framework is to find $\bw$ that maximizes the objective in Equation \ref{eq:adv_max_min_obj}. Following Equation \ref{eq:nonParametricEmpiricaLoss}, for a fixed classifier $d$, the generator's loss is linear in $\bw$ and $\frac{\partial \lossn}{\partial w_i} = l(\disc(\bx_{i}),0)$ is constant. To maximize the objective in Equation \ref{eq:adv_max_min_obj}, which refers to \emph{any} classifier from the considered family, we update the weights using a single step of exponentiated gradient ascent:
\begin{equation}
w^{t+1}_i = n \frac{ w^{t}_i \exp\big( \alpha \cdot l(\disc(x_{i}),0)\big) } {\sum_{j} w^{t}_j\exp\big( \alpha \cdot l(\disc(x_j),0)\big)}
\end{equation}
Algorithm \ref{alg:euclid} shows the complete details of the adversarial framework for non-parametric generation of balancing weights.
\begin{algorithm}[tb]
	\caption{Adversarial balancing weights}\label{alg:euclid}
	\begin{algorithmic}[1]
	\INPUT $S = \{ \bx_i\}_{i=1}^n$, $T = \{ \bx_i \}_{i=n+1}^{n+n'}$
	\PARAMETERS  classifier family $\Hfamily$, update rule for learning rate $\alpha$, number of iterations $n_{iter}$, loss function $l$
	\OUTPUT Balancing weight vector $\bw$ for $\DS$ 

	\STATE  $\lab \gets \big[\underbrace{0,0,\ldots,0}_{n\text{-times}},\underbrace{1,1,\ldots,1}_{n'\text{-times}}\big]$
	\STATE $\bw \gets \big[\underbrace{1,1,\ldots,1}_{n\text{-times}},\underbrace{1,1,\ldots,1}_{n'\text{-times}}\big]$
	\STATE $w_i \gets \frac{n}{n'}w_i \,\,\,,\forall  i > n$ \COMMENT{equal class importance}
	\FOR{$n_\text{iter}$ iterations}
			\STATE  $\widehat{\bc} \gets {\text{get\_predictions}}(\Hfamily, \left[S,T \right],\lab,\bw)$
			\STATE  $w_i \gets w_i \exp\big(\alpha_i \cdot l(\widehat{c}_i,0)\big) \,\,\,,\forall  i \leq n$ 
			\STATE $w_i \gets n \frac{w_i}{\sum_{j \in \idx_a} w_j}  \,\,\,,\forall i \leq n $
	\ENDFOR
	\STATE {\bf return} $\bw[i\leq n]$ 
	\end{algorithmic}
\end{algorithm}
Only the weights for $\DS$ are updated, while weights for the sample units in $\DT$ are constantly set to $1$. In each iteration the sum of weights in $\DS$ equals $n$, ensuring the same importance with respect to the discriminator loss.  The predictions of the discriminator (Step 6 in Algorithm \ref{alg:euclid}) should preferably be obtained with cross validation, to better approximate the generalization error in Equation \ref{eq:expectedBalancingLoss}.

The choice of the classifier’s family and its hyper parameters is important to enable us to approximate the minimal loss defined in Equation \ref{eq:expectedBalancingLoss} with the empirical loss in Equation \ref{eq:nonParametricEmpiricaLoss}. On the one hand, we would like to reduce the estimation error of $\lossn$ due to over-fitting of the classifier. On the other hand, the family of classifiers should be rich enough to distinguish between "non-similar" (weighted) datasets. In Section \ref{sec:experiments}, we describe our experiments with different classification algorithms, ranging from the low-capacity logistic regression to the large-capacity class of neural networks.  We test the ability of our framework to tackle the challenge of bias-variance trade-off by applying a preliminary step of hyper-parameter selection using cross-validation, prior to running Algorithm \ref{alg:euclid}. 

\section{Theoretical results}
In this section we provide theoretical results for the 0-1 loss function, formulating the link between the classifier family and the estimation error.  We start with introducing the two-sample divergence measure induced by the discriminator error, namely the $\Hfamily$-divergence. We use this divergence measure in the bound we provide for the estimation error. 

\subsection{The two-sample $\Hfamily$-divergence}
Let $\Hfamily$ denote the family of binary classifiers $h : X \rightarrow {0,1}$ considered by the discriminator.  Similar to \cite{kifer2004detecting,ben2007analysis, ben2010theory} we define the $\Hfamily$-divergence between $\wS$ and $T$ as:
\begin{eqnarray}
\label{eq:hdiv}
	\dHapprox(\bw) = 2 \max_{h\in \Hfamily} \left\vert    \frac{1}{n} \sum_{i=1}^n w h(x_i) -    \frac{1}{n'}  \sum_{i=n+1}^{n+n'} h(x_i)       \right\vert  \,. 
\end{eqnarray}

That is, the $\Hfamily$-divergence relies on the capacity of the hypothesis class $\Hfamily$ to distinguish between examples from $\wS$ and $T$.   Adapting a result from \cite{ben2007analysis}, the following lemma links  the $\Hfamily$-divergence between $\wS$ and $T$, with the minimal error of the  discriminator that attempts to classify them, as defined in Equation \ref{eq:nonParametricEmpiricaLoss}.

\begin{lem}
\label{lem:hdiv_Ln}
If $\Hfamily$ is symmetric, that is, for every $h \in \Hfamily$, the inverse
hypothesis $1 - h$ is also in $\Hfamily$, and the loss function $l$ in  is the $0-1$ loss, then 
\[ \dHapprox(\bw) = 2 \left[ 1- \min_{h \in \Hfamily}\lossn(\bw, h) \right] \]
\end{lem}
\begin{proof}
See appendix.
\end{proof}
In the following we assume that $\Hfamily$ is symmetric. 

\subsection{Bounds for estimation error }
Suppose that $\DS$ are assigned with labels $\{y_i\}_{i=1}^n$ corresponding to the observed treatment outcomes  $\cY{a}$.  Denote $\foutcome(X) \equiv \expect{\DT}{\cY{a}|X}$. In this section we provide a bound for the estimation error for the case where $\foutcome(\bx)$  is bounded, with $\Mfoutcome= \sup_\bx |\foutcome(\bx)|$. The estimation error can be decomposed to two sources of error by adding and subtracting terms, using $\expect{\DT}{\cY{a}} = \expect{\DT(X)}{\foutcome(X)} $ and the triangle inequality:
\begin{eqnarray}
\label{eqn:errdecomposition1}
\left\vert  \frac{1}{n}\sum_{i=1}^n w_i y_i- \expect{X \sim \DT}{\cY{a}} \right\vert  \leq &&  \\
&&\hspace{-80pt}
\left\vert  \frac{1}{n}\sum_{i=1}^n w_i y_i  - \frac{1}{n'} \sum_{i=n+1}^N \foutcome(x_i) \right\vert \nonumber \\ 
&&\hspace{-80pt}
+\left\vert    \frac{1}{n'} \sum_{i=n+1}^N \foutcome(x_i) -  \expect{X \sim \DT}{\foutcome(X)}  \right\vert \nonumber 
\end{eqnarray}
The second term, which does not depend on the weights $\bw$, relates to the approximation of the expected value of $\foutcome(X)$ by a sample mean of $\foutcome(\bx)$. Following Hoeffding's inequality it is bounded by $2 \Mfoutcome \sqrt{\frac{\ln\frac{2}{\delta}}{2n’}}$ with probability $1-\delta$.
For the remaining of this section we focus on bounding the first term, which involves the weights $\bw$. We start by decomposing this term using the triangle inequality
\begin{align}
\label{eqn:errdecomposition}
& \left\vert  \frac{1}{n}\sum_{i=1}^n w_i y_i- \frac{1}{n'} \sum_{i=n+1}^N \foutcome(x_i) \right\vert  \leq \left\vert  \frac{1}{n}\sum_{i=1}^n w_i (y_i-\foutcome(x_i))  \right\vert  \nonumber \\
&\hspace{30pt}
+\left\vert  \frac{1}{n}\sum_{i=1}^n w_i \foutcome(x_i) - \frac{1}{n'} \sum_{i=n+1}^N \foutcome(x_i) \right\vert 
\end{align}
The first term in Equation \ref{eqn:errdecomposition}  depends on the variability of the outcome $\cY{a}$ \emph{given} $X$, as well as on the weights $\bw$.   We will address this term in Theorem~\ref{thm:bound} below. The second term in this equation depends on the difference between the weighted average of $\foutcome(X)$ on $S$ and the unweighted average of $\foutcome(X)$ on $T$.  Following Equation~\ref{eq:hdiv}, if $\foutcome \in \Hfamily$ then this term is smaller or equal to half of the $\Hfamily$-divergence.  The following lemma extends this observation for a larger family than $\Hfamily$, noted as $\Comb(\Hfamily)$, which contains all functions that can be represented by a bounded linear combination of
members in $\Hfamily$.  More formally we define this larger family as $\Comb(\Hfamily)=
\{ f :  f = \sum_j  \alpha_j h_j(x) \text{ s.t. } h_j \in \Hfamily \text{  and } \sum_j |\alpha_j| \leq \Mfoutcome \}$.
\begin{lem}
Suppose that  $\foutcome \in  \Comb(\Hfamily)$. Then for every $S$ and $\bw$,
\label{lem:dh_g}
\begin{align*}
\left| \frac{1}{n}\sum_{i=1}^n w_i \foutcome(x_i)  - \frac{1}{n'}\sum_{i=n+1}^N \foutcome(x_i)\right| \leq \frac{\Mfoutcome}{2} \dHapprox(Sw, T)
\end{align*} 
\end{lem}
\begin{proof}
See appendix.
\end{proof}
Lemma \ref{lem:dh_g} leads to the following bound:
\begin{thm}
\label{thm:bound}
Given $\bw$ and $S=\{ x_i \}_{i=1}^n$, If $\foutcome \in \Comb(\Hfamily)$ then
for any $\delta \in (0,1)$, with probability of at least $1-\delta$ we have
\begin{eqnarray*}
\left\vert  \frac{1}{n}\sum_{i=1}^n w_i y_i- \frac{1}{n'} \sum_{i=n+1}^N \foutcome(x_i) \right\vert 
 \leq && \\ 
 &&  \hspace{-100pt}
 \frac{\Mfoutcome}{2} \dHapprox(Sw, T) + 2 \Mfoutcome \sqrt{2 \norm{\frac{\bw}{n}}_2^2 \ln\frac{2}{\delta}}  
\end{eqnarray*}
\end{thm}
\begin{proof}
	Define a set of random variables $Z_i \sim \dist(w_i(\cY{a}-\foutcome(x_i))|x_i)$. Each $Z_i$ is bounded by $2w_iM_Y$. 
	Applying Hoeffding's inequality \cite{mohri2018foundations}  yields that
for any $\delta \in (0,1)$, with probability at least $1-\delta$
$$\left\vert  \frac{1}{n}\sum_{i=1}^n w_i y_i - \frac{1}{n} \sum_{i=1}w_i \foutcome(x_i) \right\vert  < 2 \Mfoutcome \sqrt{2 \norm{\frac{\bw}{n}}_2^2 \ln\frac{2}{\delta}} $$
See appendix for details.
Following Equation \ref{eqn:errdecomposition},  this inequality together with Lemma \ref{lem:dh_g} provides the desired proof.
\end{proof}

The first term in the bound given in Theorem~\ref{thm:bound} corresponds to the $\Hfamily$-divergence, which is the objective that our weights generator aims to minimize.  This implies a tradeoff induced by selection of the discriminator.
Using a rich family of classifiers allows to have a good approximation of the function family $\Comb(\Hfamily)$. On the other hand a compact family leads to low $\Hfamily$-divergence and allows to avoid overfitting. Note that the empirical $\Hfamily$-divergence can be examined after running the algorithm and thus provide an indication of potential errors.

The second term in this bound is dominated by $\norm{\frac{\bw}{n}}_2^2$, indicating the variability in the weights. Observe that $\frac{\bw}{n} \in \Delta$ and that $\min_{\bu \in \Delta} \norm{\bu}_2^2 = \frac{1}{n}$. This minimum is obtained for $\bu^* = \frac{\be}{n}$, where $\be = (1,\dots,1)$. Therefore,
when the weights are close to uniform this term converges at the rate of $\sqrt{n}$.  Therefore it is desirable to maintain low variability of weights. Note that this variability is bounded by $\norm{\bu}_2^2 \leq K\! L(\bu, \frac{\be}{n})$, where $K\! L$ is the Kullback-Leibler divergence \cite{shalev2012online, beck2003mirror} (see supplemental material).  The exponentiated gradient ascent, which maintains the normalization constraint of the computed weights, has a desired property of generating weights with minimal Kullback-Leibler divergence to previous weights \cite{kivinen1997exponentiated}. Therefore, our algorithm is expected to produce weights that remain as close as possible (in an entropy sense) to the initial uniform weights. 

In the supplemental material we provide a bound for the general case in which $\foutcome$ is not necessarily in $\Comb(\Hfamily)$. In this case the bound includes an additional term corresponding to a proximity measure of $\foutcome$ to $\Comb(\Hfamily)$.  

\section{Related work}
Inverse propensity weighting (IPW) \cite{rosenbaum1987model} is a widely-used balancing method that models the conditional treatment probability given pre-treatment covariates. If the model is correctly specified, then the computed weights are balancing \cite{horvitz1952generalization}. However, a misspecified model may generate weights that fail to balance the biases, potentially leading to erroneous estimations. In recent years, various methods were developed to  generate weights that directly minimize the different objectives used to measure the discrepancy between compared populations (e.g., \cite{hainmueller2012entropy,graham2012inverse,imai2014covariate,zubizarreta2015stable,chan2016globally,kallus2017balanced,zhao2016covariate}). Each of these methods provides alternative solutions to the following elementary problems: $(i)$ how to measure the bias between two distributions, and $(ii)$ how to generate weights that minimize it. Some of the methods, (e.g., \cite{graham2012inverse} and \cite{imai2014covariate}), fit propensity score models with balance constraints, to guarantee that even if the propensity model is misspecified, these balancing constraints are met. The other algorithms, including the one presented here, are designed to minimize a selected imbalance measure without considering the related propensity scores. 

A widely used criterion for assessing the imbalance between two treatment groups is the {\it standardized difference} in the mean of each covariate \cite{rosenbaum1985constructing}. Many of the algorithms, such as \cite{hainmueller2012entropy,imai2014covariate,chan2016globally}, focus on minimizing the difference between the first-order moments of the covariates or their transformations. The algorithm in \cite{kallus2017balanced} uses the {\it maximum mean discrepancy} (MMD) measure (see \cite{gretton2007kernel} for definition), which can account for an infinite number of higher order moments based on kernel methods. Very recently, an independent study \cite{kallus2018deepmatch} presented a similar idea of using GANs to generate balancing weights. However, the discrepancy objective, the weights model, and the entire algorithm in this study differ from the ones we introduce in this paper.

The problem of finding balancing weights has been studied in the field of density ratio estimation \cite{sugiyama2012densitybook, sugiyama2012densitypaper, mohamed2016learning}.  A similar problem has also been studied in the context of transfer learning under the assumption of covariate shift, where the task is to learn a prediction model from a labeled training data drawn from a source domain different from the target domain \cite{huang2007correcting, sugiyama2008direct,mansour2009domain}.

Our algorithm, and the other weighting methods we reviewed above, balance covariates without using outcome data. The resulting balanced data may be used for subsequent causal inference analysis involving multiple outcomes. Other state-of-the-art methods for causal inference, such as BART \cite{chipman2010bart} and Causal Forests \cite{wager2017estimation} focus on training outcome models that allow causal inference for a specific outcome. Finally, there are causal inference methods that combine a treatment assignment model with an outcome model, such as the augmented inverse probability weighting (AIPW) \cite{robins1994estimation,scharfstein1999adjusting,robins2000robust,glynn2010introduction}.  
Recent works in causal inference that took this approach use deep neural networks for learning a new representation of the data that improves outcome prediction on one hand, and on the other hand minimizes the discrepancy between the source and target data \cite{johansson2016learning,shalit2017estimating}. The approach of learning a representation that minimizes the discrepancy between source and target domains in an adversarial manner, while optimizing label prediction, has recently become very popular in transfer learning, with vast applications in computer vision \cite{ganin2016domain, tzeng2017adversarial}. 

\section{Experiments}\label{sec:experiments}

We evaluated our adversarial weighting method  on three previously published benchmarks of simulated data by "plugging-in" various classifiers.
 We compared our method to IPW with the same classifier, and tested against more recent methods for balancing weights. We focused on methods that do not use information on the outcome for estimating the weights.

\subsection{Experimental setting}
The results reported in this section are based on the zero-one loss function. We considered the following "plug-in" classifiers as the discriminator:
{\bf LR}: Logistic regression (default parameters by \texttt{Scikit-learn});
{\bf SVM}: a support vector machine with RBF kernel (default parameters by \texttt{Scikit-learn});
{\bf MLP}: a multilayer perceptron with 1-3 internal layers.  The number of nodes in each internal layer is set to twice the number of variables in the input layer. The exact number of internal layers is selected  as the one that minimizes the zero-one prediction error (generalization error) evaluated in a 5-fold cross-validation procedure;
{\bf LR/SVM/MLP}: a classifier that is selected from the previously described classifiers as the one minimizing the zero-one prediction error evaluated in a 5-fold cross-validation procedure.

Note that for the classifiers MLP and LR/SVM/MLP, the configuration is set once before running the weighting algorithms.
We used a decaying learning rate $\alpha$ in Algorithm \ref{alg:euclid}: $\alpha_{t+1} = \frac{1}{1+0.5 \cdot t}$ and limited the number of iterations $T$ to 20. Finally, to speed running times we configured the function get\_predictions in Step 5 of Algorithm \ref{alg:euclid} to return train predictions. 

We compared the results of Algorithm \ref{alg:euclid} to the results obtained by the following weighting methods:
{\bf IPW}: The straightforward inverse propensity weighting, without weight trimming or other enhancements. We tested IPW with the same classifiers we used for the adversarial algorithm;
{\bf CBPS}: Covariate Balancing Propensity Score (CBPS) \cite{ imai2014covariate}, using its R package \cite{fong2014cbps};
{\bf EBAL}: Entropy balancing \cite{hainmueller2012entropy}, using its R package \cite{hainmueller2014package};
{\bf MMD-V1}, {\bf MMD-V2}: An algorithm for minimizing the maximum mean discrepancy (MMD) measure using an RBF kernel \cite{kallus2016generalized, kallus2017balanced}. In MMD-V1 the RBF scale parameter was set to 1. MMD-V2 includes a preliminary step for selecting the RBF scale and a regularization parameter \cite{kallus2016generalized}. We implemented MMD-V1 and MMD-V2 using the \texttt{quadprog} Python package.

\begin{figure*}[th]
	\subfloat[Kang-Schafer benchmark]
	{
		
		\includegraphics[scale=.47]{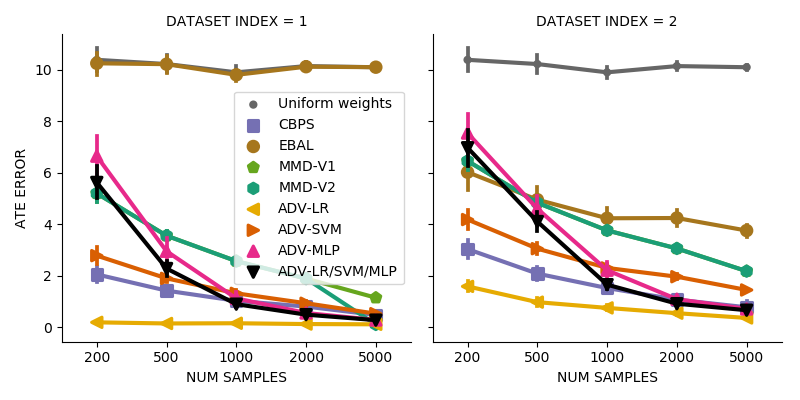}
	}
	\subfloat[Circular benchmark]
	{
		\includegraphics[scale=.46]{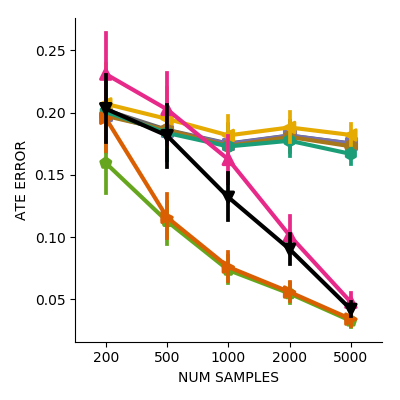}
	} 
	\caption{Comparison of  weighting algorithms: CBPS, EBAL, MMD and the adversarial algorithm. We compare the adversarial algorithm with two different classifiers: logistic regression (LR) and LR/SVM/MLP.  The latter corresponds to the adversarial algorithm with a preceding step of model selection from (i) logistic regression (LR), 
		(ii) support vector machine with RBF kernel (SVM), and (iii) multi-layer perceptrons MLP. MLP corresponds to MLPS with 1/2/3 layers, respectively chosen by cross-validation. 
		Horizontal lines represent 95\% confidence intervals computing using bootstrapping. }
	\label{fig:classifier_cmp}
\end{figure*}

\begin{figure}[t]
	\hspace{-40pt}
	\includegraphics[scale=.34]{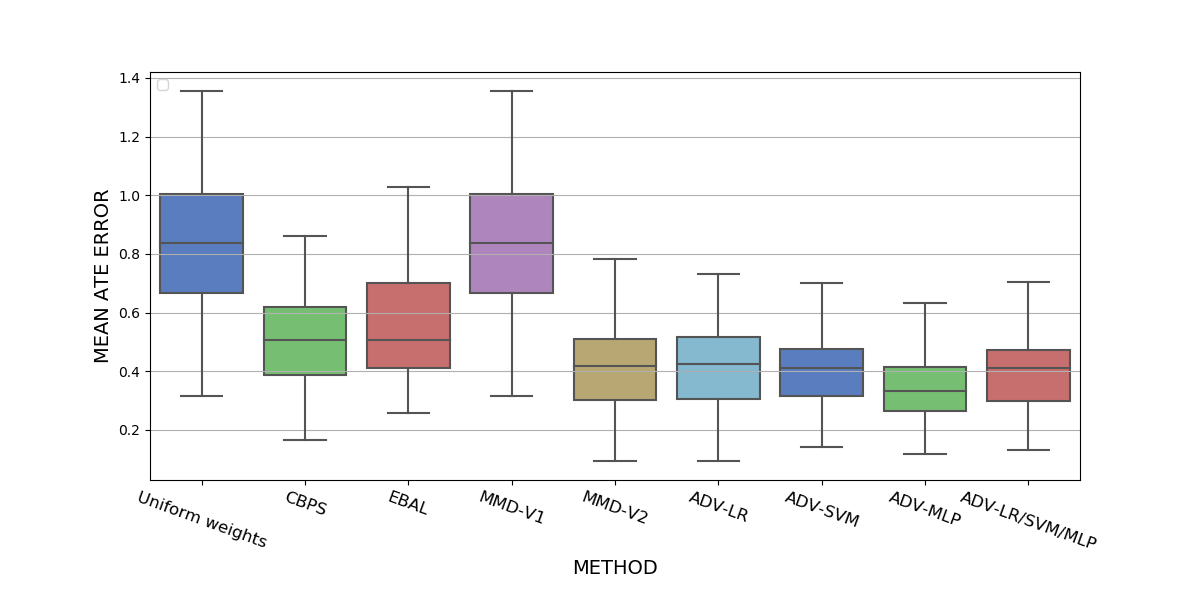}
	\caption{ Comparison on ACIC benchmark}
	\label{fig:acic_cmp}
\end{figure}

\subsection{Benchmarks}

We evaluated and compared the different weighting methods on the following benchmarks:

{\bf Kang-Schafer benchmark}  \cite{kang2007demystifying}: The data includes four independently normally distributed covariates: $X_1, X_2, X_3, X_4 \sim N(0,1)$. The outcome covariate $Y$ is generated as $Y=210+27.4 X_1 +13.7 X_2 +13.7 X_3 +13.7 X_4 +\epsilon$ where $\epsilon \sim N(0,1)$.  The true propensity score is $p(A=1|X_1, X_2, X_3, X_4) = \expit(-X_1+0.5 X_2 -0.25 X_3 -0.1 X_4)$ . The outcome $Y$ is observed only for $A=1$. The simulation includes two scenarios. In the first, the covariates $(X_1, X_2, X_3, X_4)$ are observed, while in the second the covariates actually seen, $(X'_1, X'_2, X'_3, X'_4)$, are generated as:  $X'_1 = \exp(X_1/2)$, $X'_2= X_2 / (1+\exp(X_1)) +10$, $X'_3 = (X_1*X_3/25 + 0.6)^3$, and  $X'_4 = (X_2+X_4+20)^2$. As $Y$ is observed only for a biased selection of the data, the task in this benchmark is to estimate the expected potential outcome $E(Y_1)$ for the entire population.  In this case we apply Algorithm \ref{alg:euclid} once to balance the subpopulation of $A=1$ with the entire population. We generated 5 paired datasets, where each pair corresponds to the 2 scenarios, for data size  $n=200, 500, 1000, 2000, 5000$. Each of the datasets includes 100 random replications. Paired datasets are based on the same randomized covariates $(X_1, X_2, X_3, X_4)$.

{\bf Circular benchmark}: These simulations are based on the example given in \cite{kallus2016generalized} with a minor modification to accommodate estimation of ATE. The simulations follow a scenario with two covariates $X_1$ and $X_2$ independently drawn from a uniform distribution on $[-1, 1]$. The true propensity score is $p(A=1|X_1, X_2) = 0.95/(1+\frac{3}{\sqrt{2}}{{\lVert (X_1, X_2) \rVert}_2})$. The potential outcomes $Y^0$ and $Y^1$ are independently normally distributed with means $\lVert X \rVert_2^2-X_1/2-X_2/2$ and $\lVert X \rVert_2^2$, respectively, and with a standard deviation of $\sqrt{3}$. We generated 5 datasets, each with 100 random replications, for this scenario with data size  $n=[200, 500, 1000, 2000, 5000]$.

{\bf ACIC benchmark}: The Atlantic Causal Inference Conference (ACIC) benchmark  \cite{ACIC} includes 77 datasets,  simulated with different treatment assignment mechanisms and outcome models. All the datasets use the same 58 covariates with 4802 observations derived from real-world data. These simulations accounted for various parameters, such as degrees of non-linearity, percentage of patients treated, and magnitude of the treatment effect.  Each of the 77 datasets includes 100 random replications independently created by the same data generation process, yielding 7700 different realizations in total. For a complete description of this benchmark, see \cite{ACIC}.

\subsection{Results}

For all considered classifiers, the adversarial algorithm outperformed its counterpart IPW in most of the tests, in particular on the large sample size (see Figure A.1 in the supplemental material ).

Figure \ref{fig:classifier_cmp} shows the results of Algorithm \ref{alg:euclid}, CBPS, EBAL, MMD-V1 and MMD-V2 on the Kang-Schafer and Circular benchmarks.  
As a reference, we selected two classifiers for the adversarial algorithm: LR being the simplest classifier and LR/SVM/MLP for its ability to adapt to the data.
As shown, in the Kang-Schafer benchmark, ADV-LR outperforms CBPS, EBAL, and both versions of MMD.  In the Circular benchmark, MMD-V1 and ADV-SVM outperformed all compared methods, possibly because it employs Gaussian kernels that can handle the circular contours of the propensity function. Note that the performance of all classifiers improves with data size, and becomes more similar in the final point (n=5000).

Figure \ref{fig:acic_cmp} presents the results on the ACIC benchmark. These results also support our previous observation that the adversarial framework is better at exploiting classifiers than the IPW method. This plot also provides some evidence for the robustness of the cross-validation procedure, as ADV-LR/SVM/MLP steadily remains one of the top-performing methods.  Finally, even our weakest variant ADV-LR exhibited performance superior to CBPS and EBAL, and results comparable to MMD-V2. 

We see that each benchmark had a different classifier that obtained the best results in the adversarial framework, with ADV-LR excelling in the Kang-Schafer benchmark, ADV-SVM in the Circular benchmark, and  ADV-MLP in ACIC. However, in all three benchmarks, LR/SVM/MLP was the second-best performing classifier, suggesting that it is more robust in unknown scenarios.  

\section{Discussion}
We introduced an adversarial framework for generating balancing weights, which uses a classifier family for measuring the discrepancy between two data samples, and exponentiated gradient ascent, to compute weights that minimize this divergence.  Our theoretical results for the estimation error provide further motivation for (i) obtaining weights that maximize the minimal classification error with 0-1 loss for the two samples; and (ii) using exponentiated gradient descent for generating normalized weights that optimize this objective while keeping themselves close to uniform. Our experimental results provide additional support for the effectiveness of exponentiated gradient descent in generating weights that lead to estimates with smaller variance. This setup allows us to easily plug-in a plethora of classification algorithms, each corresponding to a different classifier family,  into our framework.  

The selection of the classifier family clearly affects the sensitivity of the discriminator in identifying biases between the samples. Low capacity classifier families may weaken the ability of the discriminator to distinguish important biases, while higher capacity families will result in large discrepancy measures even when the samples are balanced.  Note that a classification algorithm may use an objective different than the discriminator's error for selecting the best classifier. In this case we assume it may still be used in practice under the assumption that the selected classifier highly correlates with the classifier that would have been selected by the discriminator.  In particular, classifiers that incorporate regularization in their objective will have less over-fitting for stronger regularization, yielding a reduced sensitivity of the discriminator, and consequently smaller discrepancy measures.  

Selecting a classifier family with an appropriate sensitivity level is a challenging task. We applied a heuristic that uses cross-validation to select the classification algorithm with the lowest estimated generalization error, under the assumption that models with larger generalization error may be either over-sensitive when the error is due to over-fitting,  or not sensitive enough when the error is due to large bias. The experiments we conducted on different benchmarks may provide a support for this approach,  as the adversarial algorithm with auto-select classifier always ranked second. A future research direction is to improve classifier selection so it reaches comparable results to the first ranked classifier.

The discrepancy measure induced by the discriminator is determined not only by the choice of the classifier family but also by the loss function. When the log loss is used and the family of classifier has enough capacity, then the induced discrepancy measure approximates the Jensen–Shannon divergence \cite{goodfellow2014generative}.   The discrepancy measures induced by the 0-1 loss, can be viewed as {\it integral probability metrics} (IPMs) , which are defined with respect to a family, $\Ffamily$, of real-valued bounded functions \cite{sriperumbudur2012empirical}: $\text{IPM}_{\Ffamily}(\DS, \DT) \equiv \left| \sup_{f \in \Ffamily} \expect{X\sim \DS}{f(X)} - \expect{X\sim \DT}{f(X)} \right|$.  Discrepancy measures that can be presented as IPMs include the Wasserstein distance (also known as Earth-Mover distance) and MMD \cite{sriperumbudur2012empirical}.  There are extensions of GANs where the discriminator is replaced by  a two sample-test corresponding to the Wasserstein distance \cite{arjovsky2017wasserstein} and the MMD distance \cite{li2015generative, dziugaite2015training,li2017mmd}.  A future work would be to adapt our framework for estimating these discrepancy measure and minimize them with exponentiated gradient descent.

\medskip

\small

\bibliography{ref}
\bibliographystyle{icml2019}

\end{document}